\documentclass[letterpaper, 10 pt, conference]{ieeeconf}  

\usepackage{multicol}
\usepackage[bookmarks=true]{hyperref}
\usepackage{float}
\usepackage[final]{graphicx}
\let\labelindent\relax
\usepackage[inline]{enumitem}
\usepackage{comment}

\usepackage{amsmath,amsfonts}

\newcommand{\br}[1]{\left\{#1\right\}}                            
\newcommand{\activecoreset}{\textsc{Active-Coreset}}

\newcommand{\REAL}{\ensuremath{\mathbb{R}}}
\newcommand{\Ourspan}{\mathrm{span}}
\newcommand{\mvecoreset}{\textsc{MVE-Coreset}}
\newcommand{\term}[1]{\left( #1 \right)}
\DeclareMathOperator*{\argmax}{arg\,max}
\newcommand{\mynorm}[1]{\left\| #1 \right\|}

\usepackage[linesnumbered,lined,ruled,noend]{algorithm2e}
\usepackage{multirow,array}
\usepackage{siunitx}
\usepackage{diagbox}
\usepackage{booktabs}
\usepackage{multirow,array}
\usepackage{siunitx}
\usepackage{rotating}
\usepackage{subcaption}
\usepackage{kantlipsum}

\usepackage{xparse}
\ExplSyntaxOn
\DeclareExpandableDocumentCommand{\oursqrt}{m}
 {
  \fp_eval:n { round ( sqrt(#1) ) }
 }
\ExplSyntaxOff

\usepackage{adjustbox}

\newtheorem{theorem}{Theorem}

\newtheorem{lemma}[theorem]{Lemma}

\newtheorem{definition}[theorem]{Definition}

\IEEEoverridecommandlockouts                              

\title{\LARGE \bf
Obstacle Aware Sampling for Path Planning
}

\author{Murad Tukan \and Alaa Maalouf \and Dan Feldman \and Roi Poranne 
\thanks{All authors are with the Computer Science Department, University of Haifa, Haifa 3498838, Israel
         {\tt\small \{muradtuk, alaamalouf12, dannyf.post, roi.poranne\}@gmail.com}}
\thanks{This work has been submitted to the IEEE for possible publication. Copyright may be transferred without notice, after which this version may no longer be accessible}
}

\begin{document}

\newcommand{\approxcoreset}{\textsc{Approx-MVE-Coreset}}
\newcommand{\far}{\textsc{Farthest}}
\newcommand{\oracle}{\mathrm{oracle}}
\newcommand{\direction}{direction}
\newcommand{\true}{\mathrm{true}}
\newcommand{\abs}[1]        {\left| #1\right|}
\newcommand{\conv}[1]{\mathrm{conv}\term{#1}}
\newcommand{\false}{\mathrm{false}}
\newcommand{\eps}{\varepsilon}
\newcommand{\vol}{\mathrm{vol}}
\newcommand{\mvee}{\mathrm{mvee}}

\newcommand{\X}{\mathcal{X}}
\newcommand{\Xf}{\mathcal{X}_{\textbf{free}}}
\newcommand{\Xo}{\mathcal{X}_{\textbf{obs}}}

\maketitle
\thispagestyle{empty}
\pagestyle{empty}

\begin{abstract}
Many path planning algorithms are based on sampling the state space.
While this approach is very simple, it can become costly when the obstacles are unknown, since samples hitting these obstacles are wasted.
The goal of this paper is to efficiently identify obstacles in a map and remove them from the sampling space.
To this end, we propose a pre-processing algorithm for space exploration that enables more efficient sampling.
We show that it can boost the performance of other space sampling methods and path planners. 

Our approach is based on the fact that a convex obstacle can be approximated provably well by its minimum volume enclosing ellipsoid (MVEE), and a non-convex obstacle may be partitioned into convex shapes.
Our main contribution is an algorithm that strategically finds a small sample, called the \emph{active-coreset}, that adaptively samples the space via membership-oracle such that the MVEE of the coreset approximates the MVEE of the obstacle.  
Experimental results confirm the effectiveness of our approach across multiple planners based on Rapidly-exploring random trees, showing significant improvement in terms of time and path length. 
\end{abstract}

\section{Introduction}
Path finding is one of the oldest problems in robotics.
The goal of path planning is to find a feasible path from an initial state to a final state, that does not collide with any obstacles.
The main challenge stems from the usually immense search space that is needed to explore, especially for the continuous problem.
A common approach is to cleverly \emph{sample} the state in hopes of finding a path. 
Indeed, sampling-based path planners, such as Rapidly-exploring Random Trees (RRTs)~\cite{lavalle1998rapidly}, and Probabilistic Road maps (PRMs)~\cite{kavraki1996probabilistic}, are a popular choice for path finding.

One limitation of sampling-based planners is that they allow sampling \emph{inside} obstacles, potentially leading to computational waste.
If the map contains very large obstacles, a path might not be found due to a limited number of iterations, or the convergence time of such path planners increases.
The problem is further exacerbated when the obstacles are not known in advance.
Mapping the obstacles out can be a considerable challenge on its own. 

The goal of this paper is to provide a \emph{pre-processing} algorithm that discovers obstacles and removes redundant areas from the state space, giving a form of \emph{conscience} to any sampling-based path planner, for faster convergence and possibly shorter generated paths.
It can potentially be used with any sampling-based path planning routine.
The challenge lies in that, contrary to the more common setting, the locations and shapes of the obstacles are hidden, and must be inferred using a membership oracle.
To do so, we efficiently bound the volume of an obstacle as soon as it is encountered, using the concept of a \emph{coreset} for a given implicit body approximation. 

\begin{figure}[htb!]
\begin{subfigure}[b]{.32\linewidth}
\centering
\includegraphics[width=\linewidth]{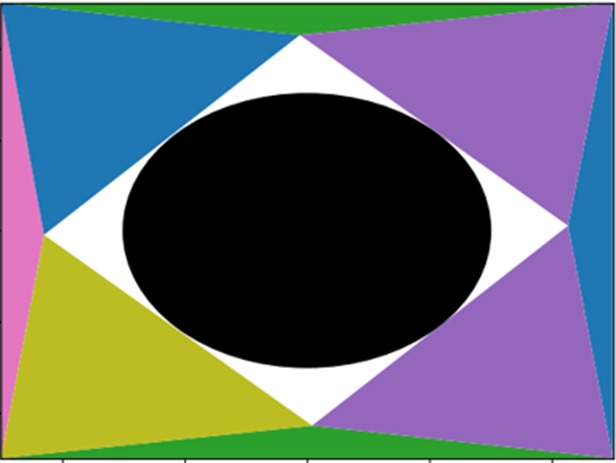}%
\end{subfigure}
\begin{subfigure}[b]{.32\linewidth}
\centering
\includegraphics[width=\linewidth]{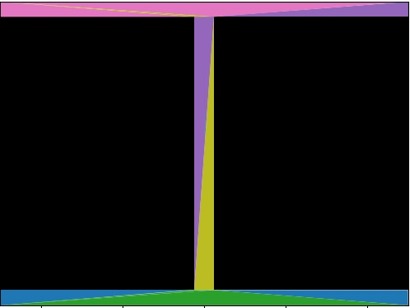}%
\end{subfigure}
\begin{subfigure}[b]{.32\linewidth}
\centering
\includegraphics[width=\linewidth]{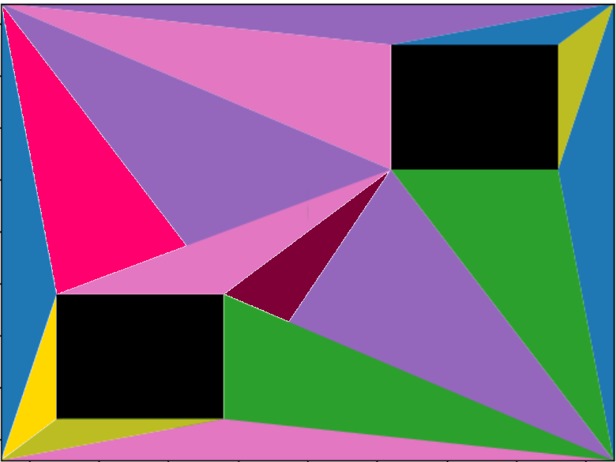}
\end{subfigure}
\caption{Partitioning the free state space using our methods. See original maps at Fig.~\ref{maps}.}
\label{fig:tri}
\end{figure}

The novelty of our approach lies in excluding explored obstacles from the state space.
This is by bounding each obstacle via a convex body using minimum calls to $\oracle$, and excluding the enclosing convex body from the state space. Our contribution is then threefold:

\begin{enumerate}
\item 
We propose a coreset for approximating the minimum volume enclosing ellipsoid (MVEE), i.e., we suggest an algorithm that computes a small subset $S$ of a given (probably infinite) set $P\subseteq\REAL^d$, such that the volume $\mathrm{vol}(\mvee(P))$ of the MVEE of $P$ is larger by a factor of at most $(1+\eps)$ of $\vol(\mvee(S))$ (Sec.~\ref{sec:coreset-details}).



\item We then extend our result towards enclosing the ellipsoid by a simplex $C$ (which is an approximation to the convex hull of the set $P$) to exclude $C$ (sets of infeasible states) from $\X$; see full details in section~\ref{covhullsec}. 

\item We define a novel technique for sampling from the new space which excludes the simplex $C$, by splitting the state space into regions using Delaunay triangulation (see Definition~\ref{def:delaunay}). Each region is chosen with probability proportional to the ratio of its volume. 
We then apply the sampling technique of the path planner from the region which has been chosen (from the previous sampling step); Fig.~\ref{fig:tri} shows the triangulation of the obstacle-free state space using our methods, and Fig.~\ref{fig:framework} sums up our methods illustratively.
\end{enumerate}

\subsection{Background and related work}
\noindent\textbf{The evolution of RRTs.} RRT-based algorithms were first proposed in~\cite{lavalle1998rapidly,lavalle2001randomized}, and ever-since they were widely leveraged by the robotics community. To ensure asymptotic optimality (of the final path), RRT* was suggested~\cite{karaman2010incremental,karaman2011sampling}, where it allows the inclusion of optimization metrics to improve the quality of the obtained solutions as the number of samples goes to infinity. Later, Informed-RRT*~\cite{gammell2014informed} was proposed as an improvement, here, when a feasible path is found, the planner starts searching for the final solution in an elliptical region inside the configuration space. 
\cite{nasir2013rrt} combine path biasing with rewiring techniques to suggest the RRT*-Smart algorithm, where, the main idea is to smooth and reduce the number of states in a founded path to its minimum number, and use these states as biases for further sampling. 
A numerous number of algorithms were suggested for improving RRT-based planners, giving rise to many variants~\cite{islam2012rrt,adiyatov2013rapidly,sintov2014time,naderi2015rt,otte2016rrtx,palmieri2016rrt,lai2019balancing}.
While most of these path planners aim to either shorten the path itself~\cite{petit2021rrt}, or to focus on certain areas leading to faster convergence~\cite{gammell2014informed}, it is hard to determine for a given map which path planner will perform better in terms of time or/and the length of the generated path.

\noindent\textbf{Coresets.} A coreset is (usually) a small weighted subset of the original input set that approximates a loss function for every feasible query up to a provable multiplicative error of $1 \pm \eps$, where $\eps \in (0, 1)$ is a given error parameter. The main idea is to be able to store data using small memory, and boost solvers by applying them on the coreset instead of the original data.
Coreset was first suggested by~\cite{agarwal2004approximating} in the context of computational geometry, and got increasing attention recently in various fields. 
For example in the context of machine learning machine learning~\cite{lucic2015strong,har2007maximum} coresets where suggested to improve the efficiency of widely used machine learning models such as regression~\cite{huggins2016coresets, munteanu2018coresets,karnin2019discrepancy}, matrix approximation~\cite{maalouf2019fast,feldman2010coresets,sarlos2006improved,maalouf2021coresets}, clustering~\cite{gu2012coreset,bachem2018one,jubran2020sets,schmidt2019fair}, $\ell_z$-regression~\cite{cohen2015lp, dasgupta2009sampling,nearconvex,sohler2011subspace}, \emph{SVM}~\cite{har2007maximum,tsang2006generalized,tsang2005core,tsang2005very,tukan2021coresets}, and  decision trees~\cite{jubran2021coresets}. 
In deep learning, the idea of coresets was leveraged for compressing deep neuronal networks~\cite{baykal2018data,liebenwein2019provable,9464761}, for robust training of neural networks against noisy labels~\cite{mirzasoleiman2020coresets} and for speeding up models training time~\cite{sinha2020small}.
computational geometry and shape approximation~\cite{agarwal2004approximating,kumar2005minimum}, and robotics~\cite{feldman2013k,nasser2020autonomous,volkov2017machine} etc.  For extensive surveys on coresets, we refer the reader to~\cite{feldman2020core, phillips2016coresets}, and to~\cite{ jubran2019introduction,maalouf2021introduction} for an introductory.

\section{Settings}
\label{ProbStat}
We first introduce our setting and necessary assumptions.

 \noindent\textbf{Obstacle.} We define an \emph{obstacle} as a convex set in $\REAL^d$. Note that a non-convex shape may be treated as the union of (hopefully few) convex sets. 
We also assume that the obstacle is not too small, otherwise, there is no benefit to finding it, as (with high probability) we should not be sampling many times in it.
More precisely, we assume that the obstacle contains a ball of radius $\eps$, for a given error parameter $\eps>0$. 

\noindent\textbf{Oracle} is a binary function $\oracle:\REAL^d\to\br{\true,\false}$ over the search space, where $\oracle(p)$ returns $\true$ if $p\in\REAL^d$ is inside an obstacle. For simplicity, we assume that a call to $\oracle$ takes $O(1)$ time, and focus on the asymptotic number of such query calls.

\noindent\textbf{Directional width.} The following definition is used to measure the width of an obstacle in a given direction. We denote by $\langle p,u \rangle$ the projection of the point $p\in \REAL^d$ on to the unit vector (direction) $u\in\REAL^d$.

\begin{definition}[$\omega(P,u)$]
For an obstacle $P$, and for any unit vector $u\in\REAL^d$, let $P[u]=\arg\max_{p\in P} \langle p,u \rangle$ be the extreme point in $P$ along $u$, then $\omega(P,u)=\langle P[u]-P[-u],u\rangle$ is called the \emph{directional width} of $P$ in the direction $u$.
\end{definition}

\noindent\textbf{Obstacles separation.} We must assume some minimal distance between obstacles, otherwise, there is no hope to distinguish between them. To keep the number of parameters small, we use $\eps$ and assume that we can expand an obstacle by a factor of $(1+\eps)$ while not hitting other objects. Formally, for every unit vector (direction) $u\in\REAL^d$, every pair of obstacles $P$ and $Q$, and for every pair of points $p\in P$ and $q\in Q$ on these obstacles, the projection $\langle p,u \rangle$  of $p$ on $u$ has distance of at least $\eps\omega(P,u)$ from the projection $\langle q,u \rangle$ of $q$ along $u$: 
$
\forall p\in P:\forall q\in Q: |\langle q,u \rangle-\langle p,u \rangle|>\eps \omega(P,u).$




\section{Method} 
Given a state space $\X$ that is composed of two sets $\X_{free}$ (the space of which the robot is allowed to pass in) and $\X_{obs}$ (the space that is covered by the obstacles), and a membership oracle $\oracle : \X \to \br{0,1}$ where for every $x \in \X$, $\oracle(x) = 0$ translates to $x \in \Xf$ and $1$ otherwise, the objective is to incrementally remove states from $\X$ that lie in $\Xo$.
The motivation is to assist the path planner cover more states in $\Xf$ to produce much better paths. This is extremely helpful when the obstacles are large and cover a lot of space.
We apply our algorithm as a preprocessing step, a sketch of it is given as follows.

\begin{enumerate}[label=(\roman*)]
    \item $x :=$ a sampled point from $\X$. \label{sketch:step1}
    \item If $\oracle(x) = 0$ ($x\in \Xf$), then go-to~\ref{sketch:step1} \label{sketch:step2}. Otherwise ($x\in \Xo$), our algorithm is invoked as follows:
    \begin{enumerate}[label=(\alph*)]
        \item $O :=$ compute a simplex which bounds the obstacle that contains $x$ using minimal calls to $\oracle$. \label{step1ofouralg}
        \item Remove the space that is covered by $O$ from the sampling space ($\X := \X \setminus O$) and go-to Step~\ref{sketch:step1}  \label{step2ofouralg}
    \end{enumerate}
\end{enumerate}
Note that our algorithm can also be applied on the fly during a normal run of the path planner, where the point $x$ from Step~\ref{sketch:step1} is sampled during execution. 

We define a convex hull of a set as follows:

\begin{definition}[convex hull]
\label{def:convexHull}
Let $P \subseteq \X$ be a (possibly infinite) set of points. Then, $\conv{P}$ is defined to be a subset of $P$ such that every $p\in P$ can be represent as a convex combination of the points in $\conv{P}$, i.e., for every $p \in P$, there exists $\Phi : \conv{P} \to [0,1]$ such that $\sum\limits_{q \in \conv{P}} \Phi(q) = 1$ and $p = \sum\limits_{q \in \conv{P}} \Phi\term{q} q$.
\end{definition}

\begin{figure}[htb!]
\includegraphics[width=\linewidth,height=.4\linewidth]{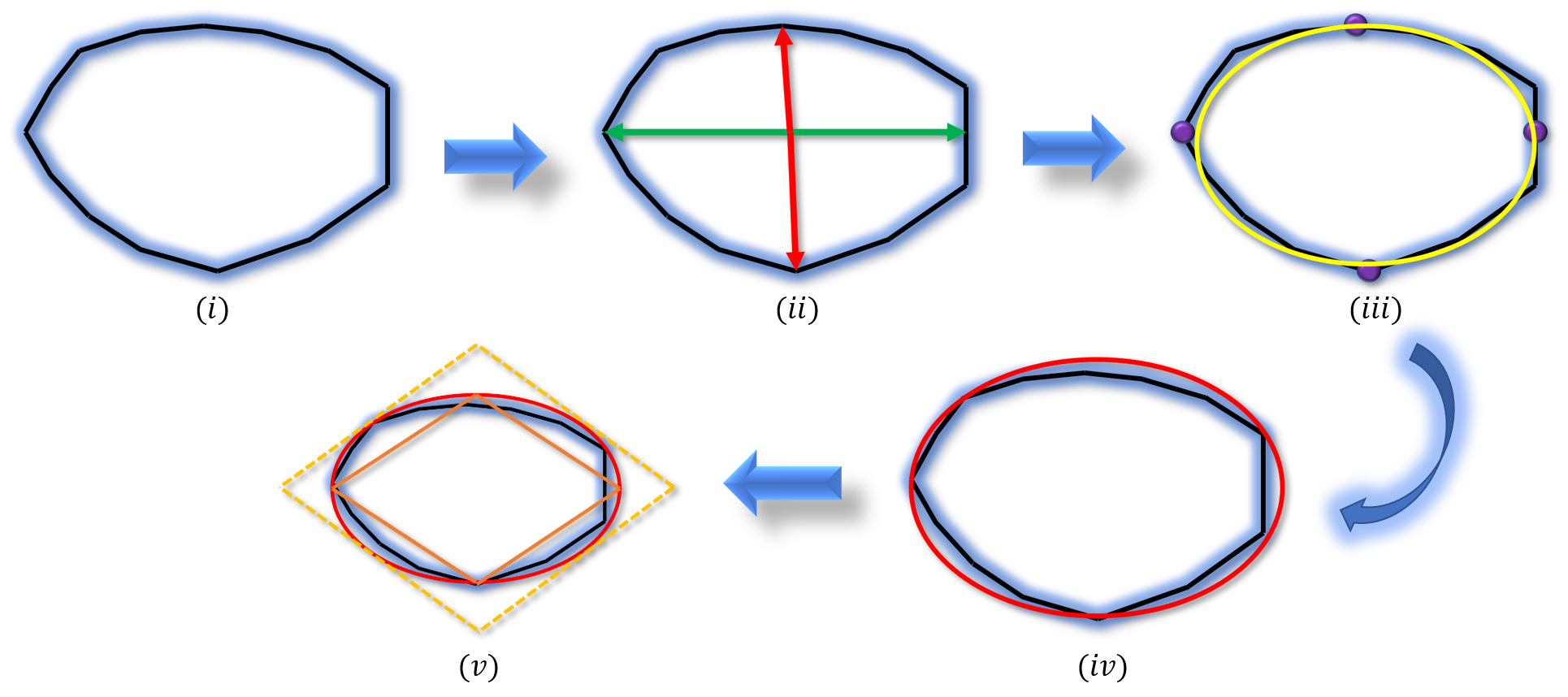}
\caption{Framework illustration for $d=2$: (i) a point is sampled from an obstacle (the blue shape), (ii) compute a $2d=4$ external points along $d=2$ orthogonal directions (the edges of the red lines and green line), (iii) compute the ellipsoid (in yellow) which passes through the set of $2d=4$ external points, this is not the minimum volume enclosing ellipsoid but a crude approximation to it, (iv) iteratively update the ellipsoid until it admits an $\eps$-approximation and finally (v) compute a bounding simplex with $O(d^{1.5}(1+\eps))$ approximation towards the volume of the convex body.}
\label{fig:framework}
\end{figure}

\noindent\textbf{Problem statement.} Given an infinite set of points $P \subseteq \X$ which is accessed using a membership oracle $\oracle : \X \to \br{0,1}$, the objective is to find a set $C \subseteq P$, $v \in P$ and some $\alpha \in [1, d]$ such that $\conv{C} \subseteq \conv{P} \subseteq \conv{\alpha^{1.5}\term{C - v} + v}.$
\subsection{Bounding obstacles - simple case: $d=1$\label{sec1}}
In this section, we give an overview of our method's execution on a one-dimensional space (the interval between $0$ and $1$). In this case the obstacles are linear segments on a line, and $\oracle:[0,1]\to\br{\true,\false}$ gets as input a scalar. While this scenario is not interesting by itself as there will be no path in the presence of an obstacle, it will help illustrate the main ideas of our core algorithm. 
Assume we sampled a point that lies inside the obstacle - the goal is to bound this obstacle in order to remove it from the sampling space (Steps~\ref{step1ofouralg} and~\ref{step2ofouralg}).
Once an obstacle is hit, we would like to find the extreme points (edge) of this obstacle based on a point $p$ inside the obstacle.
To do that, we run an exponential search (geometric sequence) of queries $p\pm 2^i\eps$, where $i=1,2,3,\cdots$. Here we used both assumptions from Section~\ref{sec1}: (i) that the minimum length of an obstacle is $\eps$, and (ii) each obstacle can be expanded by a multiplicative factor of $1 + \eps$ without hitting any other obstacles. These assumptions complete the correctness of the exponential search.
The number of iterations until the oracle returns $\false$, i.e., the first query point outside the obstacle is at most $\ln(1/\eps)$.
Using this search, we obtain a point $q$ that is outside the obstacle and of distance at most $x$ from its edge, where $x$ is the length of the obstacle.
We can then run a binary search on the interval between the outer point $q$ and the closest point to $q$ inside the obstacle that was returned by the oracle.
An $\eps$-precision of the actual extreme point of the object can be computed in additional $O(\ln(1/\eps))$ queries to the oracle.
We then repeat this binary search on each side of the obstacle.
The above process computes an $\eps$-approximation to the boundaries (convex hull) of the obstacle up to $\eps$-error which serves as our coreset.
It is easy to verify that the number of queries in each stage is minimal up to a constant factor that can be arbitrarily improved by changing the base of the log in the search. 





\subsection{Active-Learning MVEE Coresets for $d \geq 2$}\label{sec:coreset-details}

The one dimensional case $d=1$ is very simple and unique since every obstacle has exactly two boundaries or extreme points.
Already in $d=2$ dimensions, each obstacle may have many extreme points. Thus, we need to use more clever techniques that are strongly related to the minimum volume enclosing ellipsoid, known as L\"{o}wner's ellipsoid.

\begin{theorem}[L\"{o}wner Ellipsoid~\cite{ball1992ellipsoids}]
\label{thm:lowner}
Let $L$ be a convex body in $\REAL^d$, let $v \in L$, and let $E$ be ellipsoid of minimal $d$-dimensional volume containing $P$ that is centered at $v$. Let $\frac{1}{d} \term{E - c} + c$ denote the shrinkage of $E$ by a factor of $\frac{1}{d}$ around its center $c$.
Then
$
\frac{1}{d} \term{E - c} + c \subseteq L \subseteq E.
$
\end{theorem}

In the following subsections, we present our technique for computing an $\eps$-coreset with respect to the MVEE problem.

\begin{definition}[coreset for MVEE~\cite{kumar2005minimum}\label{coreelli}\label{MVEECoreset}]
For $\eps>0$ and a set $X\subseteq\REAL^d$, the set $S\subseteq X$ is an \emph{$\eps$-coreset for the MVEE} (minimum volume enclosing ellipsoid) of $X$, if the volume $\mathrm{vol}(\mvee(X))$ of the MVEE of $X$ is larger by a factor of at most $(1+\eps)$ from the volume $\vol(\mvee(S))$ of the MVEE of $S$, i.e.,
$\vol(\mvee(X))\leq (1+\eps)\vol(\mvee(S)).$
\end{definition}

Our MVEE coreset construction algorithm is based on three basic components: A) In Section~\ref{sec:findetremepoint} we suggest an algorithm for finding an extremal point on the obstacle in a specific direction. B) At section~\ref{sec:cudeaprox}, we crudely approximate the smallest enclosing ellipsoid of the obstacle, by utilizing the farthest ($2d$) points on the obstacle in a $d$ orthogonal directions, in order to construct a basis for the ellipsoid. C) Finally at Section~\ref{sec:coresetforelip}, we iteratively update the ellipsoid from the previous step, using a variant of Algorithm~3 of~\cite{kumar2005minimum} where points are accessed via an oracle.

\subsection{Finding extremal points of an implicit convex body}\label{sec:findetremepoint}

First, we give Algorithm~\ref{algtwo} that gets as input a membership oracle, an error parameter $\eps\in(0,1)$ a direction (unit vector) $u$, and a point $p$ inside an obstacle. It returns an $\eps$-approximation to the farthest obstacle point $q=p+au$ from $p$ along  $u$, i.e., $a$ is the supremum of the set $\br{a\geq 0\mid \oracle(p+au)=\true}$. This algorithm will serve as a key component in obtaining a crude approximation towards the convex hull of the infinite set of points.

\noindent\textbf{Overview of $\far(\oracle,\eps,u,p)$.} At Line~\ref{D1}, we define an arbitrary orthonormal base of $\REAL^d$, whose last vector is $e_d=u$. For simplicity, assume that $d=3$ and $e_1, e_2, e_3$ are the $x$, $y$ and $z$-axis of $\REAL^3$ respectively. Line~\ref{D2} defines a function that gets a point $(x,y)=(x_1,x_2)$ on the $xy$-plane and returns the height $f(x,y)$ of the highest obstacle point whose projection is $(x,y)$, i.e, the obstacle point $(x,y,z)$ with the maximum value of $z$. An $\eps$-approximation $\tilde{f}(x,y)$ for $f(x,y)$ with respect to the obstacle can be computed using one-dimensional binary/exponential search along the $z$-axis, as explained in Section~\ref{sec1}. The initial point is defined in Line~\ref{D4} as the $(x,y)$-coordinates of the input point $p=(x,y,z)$. At Lines~\ref{D6}--\ref{D7}, we compute the highest point whose projection is $(x_1,y)$ over every $x_1\in\REAL$, using $\tilde{f}$ above at the first iteration of the for loop. In the second (and in this example, last) iteration of the for loop, we compute the highest point $q$ whose projection is $(x_1,y_1)$ over every $y_1\in\REAL$. The height of this point is $z_1=\tilde{f}(x_1,y_1)$. We output this point $q=(x_1,y_1,z_1)$.


\begin{algorithm}
    \caption{$\far(\oracle,\eps,u, p)$\label{algtwo}}
{\begin{tabbing}
\textbf{Input:} An oracle $\oracle$ over $\REAL^d$, an error parameter \\
$\eps\in(0,1)$, a unit vector $u\in\REAL^d$, and an obstacle point\\
$p\in\REAL^d$, i.e, $\oracle(p)=\true$.\\
\textbf{Output:} An $\eps$-approximation to the farthest obstacle\\ point from $p$ along $u$.
\end{tabbing}}

\label{D1}Compute an orthonormal base $e_1,\cdots,e_d$ of $\REAL^d$, such that $e_d=u$.\\
Let $f:\REAL^{d-1}\to \REAL$ such that$
f(x_1,\cdots,x_{d-1}):=\max_{\oracle(\sum_{i=1}^d x_ie_d)=\true}x_d.$\label{D2}\\
Compute a function $\tilde{f}:\REAL^{d-1}\to \REAL$ that returns an $(\eps/d)$-approximation to $f(x)$ along $e_d$. 
\tcc{Using binary search, i.e., $O(\log(d/\eps))$ calls to $\oracle$ from the $d=1$ case with $\eps := \eps$.} \label{D3}
\While{not converged}
{
 Set $x_j:=\langle p, e_j \rangle$ for every $j \in [d-1]$ \label{D4} \\
 \For{$i:=1$ to $d-1$\label{D6}}{
 Compute an $(\eps/d)$-approximation $x_i$ along $e_i$  $\arg\max_{x_i\in\REAL} \tilde{f}(x_1,\cdots,x_{d-1})$ using oracle \label{D7}}
 $x_d:=\tilde{f}(x_1,\cdots,x_{d-1})$ \\
 }
 \Return $q := \sum_{i=1}^d x_i e_i$
\end{algorithm}





\begin{algorithm}
\caption{\approxcoreset($\oracle,\eps,p$)\label{algthree}}
{\begin{tabbing}
\textbf{Input:} \=An oracle over $\REAL^d$, an error parameter $\eps\in(0,1)$,\\
\> an obstacle point $p\in\REAL^d$, i.e, $\oracle(p)=\true$\=.\\
\textbf{Output:} \=A $O(2^{d})$-coreset $S$ for the obstacle.
\end{tabbing}}
$Q:=\br{(0,\cdots,0)}$; $S\gets \emptyset$; $i:= 0$ \\
\While{$\Ourspan(Q)\neq \REAL^d$\label{B2}}{
$i:=i+1$ \\
$x:=$ an arbitrary unit vector that is orthogonal to $\Ourspan(Q)$\\
$u:=\far(\oracle, \eps, x,p)$\label{B5}\\
$v:=\far(\oracle, \eps,-x,p)$\label{B6}\\
$S:= S\cup \br{u,v}$\label{B7}\\
$Q:= Q\cup \br{v-u}$\label{B8}\\
}
\Return $S$
\end{algorithm}
\begin{algorithm}[!htb]
    \caption{\mvecoreset($\oracle,\eps,p$)\label{algfour}}
{\begin{tabbing}
\textbf{Input:} \=An oracle over $\REAL^d$, an error parameter $\eps\in(0,1)$,  \\\>and an obstacle point $p$.\\
\textbf{Output:} \=An $\eps$-coreset $S$ for the obstacle.
\end{tabbing}}
$\hat{S} := $ \approxcoreset($\oracle,\eps, p$) \\
Let $L\in\REAL^{d\times d}$ and $c\in\REAL^d$ be the basis and center respectively of the MVEE of $\tilde{S}$\\
$S:=\emptyset$; $c_1:=c$; $L_1:=L$; $Q_i:= L_1^TL_1$ \\
\For{$i:=1$ to $\lceil d/\eps\rceil$ \label{algmain:for}}{
Compute an $(d /\eps)$-approximation to $p_{i+1}\in \argmax_{\br{q: \oracle(q)=\true}} \mynorm{L_i(q-c_i)}^2$ via Mixed Integer Convex Programming.\label{probOpt}\\
$\beta_{i+1}:= d\mynorm{L_i(p_{i+1}-c_i)}^2+1$ \\
$c_{i+1}:= (1-\beta_{i+1})c_i+\beta_{i+1} p_{i+1}$\\
$Q_{i+1} := (1-\beta_{i+1})Q_i + d\beta_{i+1} (p_{i+1}-c_i)(p_{i+1}-c_i)^T$\\
$L_{i+1} :=$ the Cholesky Decomposition of $Q_{i+1}^{-1}$\\
$S:=S\cup \br{p_{i+1}}$\label{E1}
}
\Return $S$\label{C12}
\end{algorithm}

\subsection{Crude approximated MVEE using membership oracle}\label{sec:cudeaprox}
We now provide an algorithm which when given a membership oracle $\oracle$ and an obstacle point $p$, returns an $O(2^d)$-coreset $S$ to the minimum volume enclosing ellipsoid (MVEE for short) of the obstacle that contains the point $p$, using a small number of calls to $\oracle$. The error parameter $\eps \in (0,1)$ defines the desired accuracy from the $\oracle$, during the calls to the algorithm $\far$.

\noindent\textbf{Overview of $\approxcoreset(\oracle,\eps,p)$}.
At Lines~\ref{B5}--\ref{B6}, we compute the ``leftmost'' and ``rightmost'' points $u$ and $v$ inside the obstacle along a unit vector $x \in \REAL^d$. More precisely, it computes an $\eps$-approximation to these points using the oracle and the procedure $\far$ that was previously described. We then add the points $u$ and $v$ to the coreset (Line~\ref{B7}). At Line~\ref{B2} we repeat the search on the orthogonal space of $Q$ which is iteratively updated at Line~\ref{B8}.


\subsection{$\eps$-coreset for the MVEE of an implicit convex body} \label{sec:coresetforelip}

In this subsection we describe our main technical result: an efficient construction of an $\eps$-coreset with respect to the MVEE problem, using only a membership oracle. The construction is off-line in the sense that the space (and oracle) are unchanged over time, and the computation is not parallel. Our algorithm can work in parallel using the merge-and-reduce technique~\cite{feldman2020core}. The algorithm gets as input a membership oracle, an error parameter $\eps$ and an obstacle point $p \in \oracle$ ($\oracle(p) = $\textit{true}), and returns an $\eps$-coreset to the MVEE of a given implicit obstacle via reduction to a mixed-integer convex programming.

\noindent\textbf{Overview of $\mvecoreset(\oracle,\eps,p)$.} The $i$th iteration of the for loop at Line~\ref{algmain:for}, uses an ellipsoid that approximates the obstacle, centered at $c_i$ and is defined by the affine transformation (matrix) $L_i$. This ellipsoid is the MVEE of the current coreset $S$. The approximation is improved in Line~\ref{E1} by adding a point $p_{i+1}$ to $S$. The point $p_{i+1}$ is computed in the $i$th iteration and is the farthest point in the obstacle from the current ellipsoid (defined by $L_i$ and $c_i$) vie the Mahalanobis distance. At Line~\ref{C12} we output a $\eps$-coreset for the MVEE of the obstacle.

\subsection{From ellipsoids to enclosing simplices}\label{covhullsec}
To simply the exclusion of obstacles from the state space, we further enclose our enclosing ellipsoid by a simplex; See Section~\ref{HowtoSample} for more details.

\begin{theorem}
\label{thm:boundsSimplices}
Let $P \subseteq \REAL^d$, be infinite set of points, and let $S \subseteq P$ such that 
$\mvee\term{\conv{S}} \subseteq \term{1+\eps} \mvee\term{\conv{P}}.$
Let $E$ be $\mvee\term{\conv{S}}$ that is centered at some $v \in \conv{P}$. Let $C$ be the set of $2d$ vertices of the expanded ellipsoid $\sqrt{d} \term{E - v} + v$. Then,
$\frac{1}{d^{1.5}} \term{\conv{C} - v} + v \subseteq \conv{P} \subseteq \conv{C}.$
\end{theorem}

\begin{proof}
By Theorem~\ref{thm:lowner}, it holds that 
$\frac{\term{1 + \eps}}{d} \term{E - v} + v \subseteq \conv{P} \subseteq \term{1 + \eps} \term{E - v} + v.$ By symmetry of $E$ around $v$, it holds by~\cite{ball1992ellipsoids} that
$\frac{1}{\sqrt{d}} \term{C - v} + v \subseteq E \subseteq C.$ Combining all of the inclusions above yields Theorem~\ref{thm:boundsSimplices}.
\end{proof}

\subsection{How to sample?}\label{HowtoSample}
In the following, we will discuss our approach to removing obstacles from the state space to ensure no redundancy in repeated sampling from obstacles. 

\noindent\textbf{Removing simplices from the state space.} Post to enclosing an obstacle with a simplex, we remove the simplex from the state space as follows. This objective is an easy task since we need to formulate the resulted state space. One way which we took to heart is to triangulate the resulted state space via constructing Delaunay triangulation. 

\noindent\textbf{Region sampling.} Since we have regions that were generated via the construction of Delaunay triangulation on the state space, then the probability of sampling from any region is equal to its volume divided by the sum of volumes over every region in the triangulated state space $\X$.

 \noindent\textbf{Sampling from inside a region.} Post to choosing some region (probabilistically), we apply the planner's own sampler on the region to obtain the next point for the planner.

 \noindent\textbf{Upon discovering new obstacles.} We will enclose the obstacle by a simplex as discussed in the previous section. Instead of directly applying triangulation, we find the intersection between the simplex and the current regions that represent the Delaunay triangulation of the current state space $\X$. For this, we use the Gilbert–Johnson–Keerthi distance algorithm~\cite{cameron1997enhancing} to find all intersecting regions with our new discovered simplex. We then remove our simplex from the intersecting regions followed by constructing the Delaunay triangulation on the result of such removal. This is faster than computing the Delaunay triangulation from scratch. 

\begin{definition}[Delaunay triangulation]
\label{def:delaunay}
For a set $Q \subseteq \REAL^d$, A triangulation $T(Q)$ is a partitioning of the interior of the $\conv{Q}$ into simplices, the vertices of which are points in $Q$. A Delaunay triangulation for a set $Q$ is a triangulation $T(Q)$ such that no point in $Q$ is inside the circum-hypersphere of any simplex in $T(Q)$.
\end{definition}

\section{Experimental Results}
\label{sec:results}
\subsection{Boosting the performance of RRT-based path planners}

\begin{table}[]
\begin{center}
\caption{Results for map~\ref{fig:map_c}}\label{table:mapc}
\adjustbox{max width=\linewidth}{
\begin{tabular}{|cc|cc|c|cc|}
\hline
\multicolumn{2}{|c|}{\multirow{2}{*}{\diagbox{Planner}{ Measure}}} & \multicolumn{2}{c|}{Time} & {\% of wasted} & \multicolumn{2}{c|}{Path length} \\ \cline{3-4}\cline{6-7}
& & mean & std & sampled points & mean & std\\
\hline\hline
\multirow{2}{*}{RRT} & \multicolumn{1}{|c|}{Vanilla} & $1.1$ & \oursqrt{0.28} & $81$ & $1983$ & $\oursqrt{1489}$ \\ \cline{2-7}
& \multicolumn{1}{|c|}{Our} & $\mathbf{0.36}$ & \oursqrt{0.008} & $\mathbf{0}$ &  $\mathbf{1977}$ & \oursqrt{1799}\\\hline 
\multirow{2}{*}{RRT$\ast$} & \multicolumn{1}{|c|}{Vanilla} & $0.657$  & \oursqrt{0.094}  & $68$ & $1981$ &  \oursqrt{1139} \\ \cline{2-7}
& \multicolumn{1}{|c|}{Our} & $\mathbf{0.3}$ & \oursqrt{0.02} & $\mathbf{0}$  & $\mathbf{1972}$  & \oursqrt{1890} \\\hline 
\multirow{2}{*}{RRT Dubins} & \multicolumn{1}{|c|}{Vanilla} & $93$  & \oursqrt{277}  & $77$ & $2272$ &  \oursqrt{115331} \\ \cline{2-7}
& \multicolumn{1}{|c|}{Our} & $\mathbf{48}$ & \oursqrt{74.6} & $\mathbf{0}$  & $\mathbf{2240}$  & \oursqrt{126387} \\\hline  
\end{tabular}
}
\end{center}
\end{table}

\begin{table}[]
\begin{center}
\caption{Results for map~\ref{fig:map_d}}\label{table:mapd}
\adjustbox{max width=\linewidth}{
\begin{tabular}{|cc|cc|c|cc|}
\hline
\multicolumn{2}{|c|}{\multirow{2}{*}{\diagbox{Planner}{ Measure}}} & \multicolumn{2}{c|}{Time} & {\% of wasted} & \multicolumn{2}{c|}{Path length} \\ \cline{3-4}\cline{6-7}
& & mean & std & sampled points & mean & std\\
\hline\hline
\multirow{2}{*}{RRT} & \multicolumn{1}{|c|}{Vanilla} & $81$ & \oursqrt{563.85} & $79$ & $7084$ & $\oursqrt{2087.3}$ \\ \cline{2-7}
& \multicolumn{1}{|c|}{Our} & $\mathbf{47}$ & \oursqrt{99.69} & $\mathbf{0}$ &  $\mathbf{6938}$ & \oursqrt{952.17}\\\hline 
\multirow{2}{*}{RRT$\ast$} & \multicolumn{1}{|c|}{Vanilla} & $117$  & \oursqrt{388}  & $67$ & $7066$ &  \oursqrt{606.61} \\ \cline{2-7}
& \multicolumn{1}{|c|}{Our} & $\mathbf{91}$ & \oursqrt{641} & $\mathbf{0}$  & $\mathbf{6925}$  & \oursqrt{1133.3} \\\hline 
\multirow{2}{*}{RRT Dubins} & \multicolumn{1}{|c|}{Vanilla} & $4503$  & \oursqrt{331983}  & $75$ & $6386$ &  \oursqrt{1484230} \\ \cline{2-7}
& \multicolumn{1}{|c|}{Our} & $\mathbf{4209}$ & \oursqrt{147822} & $\mathbf{0}$  & $\mathbf{4447}$  & \oursqrt{908380} \\\hline  
\end{tabular}
}
\end{center}
\end{table}

In this section, we tested the effectiveness of our approach by improving $3$ variants of the RRT algorithm on $4$ different maps. 
The idea was to apply a single prepossessing on each map, to obtain improvements for all of the RRT variants, either in terms of path length or in terms of running time.
\begin{figure*}[h]
    \centering
    \begin{subfigure}{.17\linewidth}
    \includegraphics[width=\linewidth,height=0.6\linewidth]{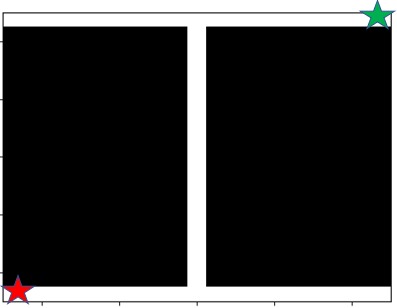}
    \caption{}  
    \label{fig:map_c}
    \end{subfigure}
    \begin{subfigure}{.17\linewidth}
    \includegraphics[width=\linewidth,height=0.6\linewidth]{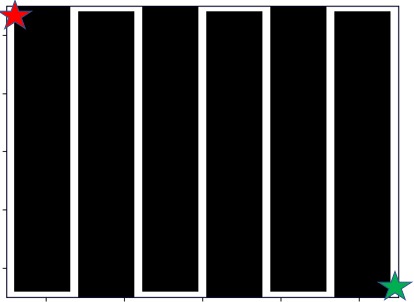}
    \caption{}  
    \label{fig:map_d}
    \end{subfigure}
    \begin{subfigure}{.17\linewidth}
    \includegraphics[width=\linewidth,height=0.6\linewidth]{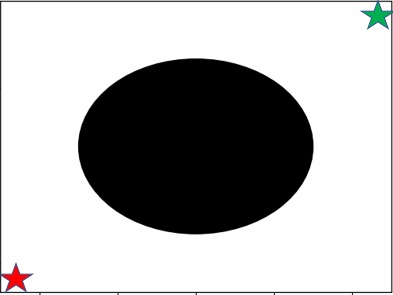}
    \caption{} 
    \label{fig:map_a}
    \end{subfigure}
    \begin{subfigure}{.17\linewidth}
    \includegraphics[width=\linewidth,height=0.6\linewidth]{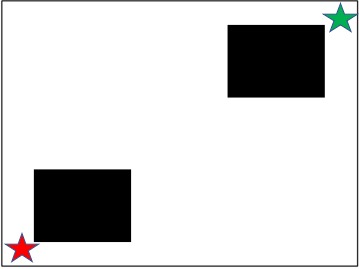}
    \caption{}
    \label{fig:map_b}
    \end{subfigure}
    \begin{subfigure}{.17\linewidth}
      \includegraphics[width=\linewidth,height=0.6\linewidth]{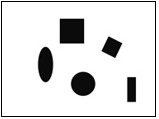}
      \caption{}
      \label{fig:univerise}
    \end{subfigure}
    \caption{Maps}
    \label{maps}
\end{figure*}
\subsubsection{Unlimited steps experiments.}  
We ran the RRT algorithms for a large number of sampling iterations on the maps~\ref{fig:map_c} and~\ref{fig:map_d}, once with the vanilla sampling technique, and once with our sampling after applying the preprocessing. We compared the following: 
\begin{enumerate*}[label=(\roman*)]
    \item the time needed for finding a solution, \label{att1}
    \item the ratio percentage of sampled points from obstacle from the total number of samples, and finally \label{att2}
    \item the length of the generated path. \label{att3}
\end{enumerate*}
Each test was conducted across $20$ different trails, the mean and variance of~\ref{att1}--\ref{att3} where reported at Table~\ref{table:mapc} and~\ref{table:mapd}. In the captions of each table, we refer to its corresponding map.  

 \noindent\textbf{Discussion. } Our proposed preprocessing technique has boosted the RRT algorithms while resulting in shorter paths from the start state to the goal state. There is a significant gap between our performance and the vanilla algorithms either in the time it took (e.g., on RRT and RRT* in both tables~\ref{table:mapc} and~\ref{table:mapd}) or in the path size (e.g., on RRT Dubins at Table~\ref{table:mapd}).
 
This is because our preprocessing ensured that the sampler will only sample non-obstacle points. Thus, the generated tree by RRT and its variants will be larger in size and will contain much more informative paths between any two states. To illustrate the advantage of our approach, Fig.~\ref{fig:trees} shows that our preprocessing done for the RRT algorithm leads to less sampling of points to attain a path from start to goal states than running plain sampling techniques.

\subsubsection{Performance under restricted number of steps} 

\begin{table}[htb!]
\begin{center}
\caption{Results for map~\ref{fig:map_a}}\label{table:map_a}
\begin{tabular}{|cc|c|cc|}
\hline
\multicolumn{2}{|c|}{\multirow{2}{*}{\diagbox{Planner}{ Measure}}} & {$\%$ of wasted} & \multicolumn{2}{c|}{Path length} \\ \cline{4-5}
& & sampled points & mean & std\\
\hline\hline
\multirow{2}{*}{RRT} & \multicolumn{1}{|c|}{Vanilla} & $26$ & $\infty$ & nan \\ \cline{2-5}
& \multicolumn{1}{|c|}{Our} & $\mathbf{0}$ &  $\mathbf{1807}$ & \oursqrt{3775}\\\hline 
\multirow{2}{*}{RRT$\ast$} & \multicolumn{1}{|c|}{Vanilla} &  $23$ & $\infty$ &  nan \\ \cline{2-5}
& \multicolumn{1}{|c|}{Our} & $\mathbf{0}$  & $\mathbf{1768}$  & \oursqrt{2423.59} \\\hline 
\multirow{2}{*}{RRT Dubins} & \multicolumn{1}{|c|}{Vanilla} & $26$ & $1771$ &  \oursqrt{13254} \\ \cline{2-5}
& \multicolumn{1}{|c|}{Our} & $\mathbf{0}$  & $\mathbf{1757}$  & \oursqrt{11054} \\\hline  
\end{tabular}
\end{center}
\end{table}

\begin{table}[htb!]
\begin{center}
\caption{Results for map~\ref{fig:map_b}}\label{table:map_b}
\begin{tabular}{|cc|c|cc|}
\hline
\multicolumn{2}{|c|}{\multirow{2}{*}{\diagbox{Planner}{ Measure}}} & {$\%$ of wasted} & \multicolumn{2}{c|}{Path length} \\ \cline{4-5}
& & sampled points & mean & std\\
\hline\hline
\multirow{2}{*}{RRT} & \multicolumn{1}{|c|}{Vanilla} & $25$ & $\infty$ & nan \\ \cline{2-5}
& \multicolumn{1}{|c|}{Our} & $\mathbf{0}$ &  $\mathbf{4416}$ & \oursqrt{57460}\\\hline 
\multirow{2}{*}{RRT$\ast$} & \multicolumn{1}{|c|}{Vanilla} & $26$ & $\infty$ &  nan \\ \cline{2-5}
& \multicolumn{1}{|c|}{Our} & $\mathbf{0}$  & $\mathbf{4404}$  & \oursqrt{63073} \\\hline 
\multirow{2}{*}{RRT Dubins} & \multicolumn{1}{|c|}{Vanilla} & $14$ & $4553$ &  \oursqrt{417577} \\ \cline{2-5}
& \multicolumn{1}{|c|}{Our} & $\mathbf{0}$  & $\mathbf{4100}$  & \oursqrt{18320} \\\hline  
\end{tabular}
\end{center}
\end{table}

\begin{figure}[htb!]
    \centering
    \includegraphics[width=\linewidth]{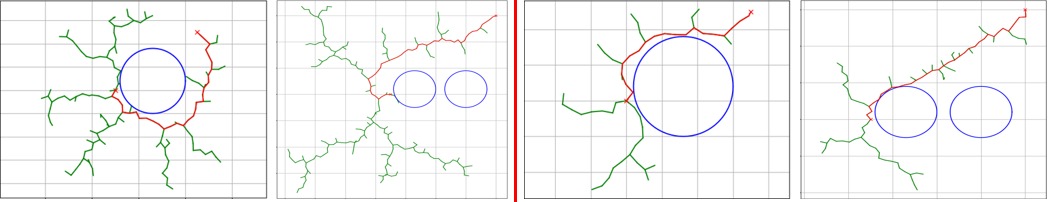}%
    \caption{Left: the generated trees using vanilla RRT. Right: The tree generated using RRTs with our preprocessing.}
    \label{fig:trees}
\end{figure}


In this experiment, we highlight the goal and motivation from which our methods have emerged. We operate under the assumption that the number of sampling iterations is restricted and small. Mostly, when using path planners, one can not know beforehand the number of iterations needed for generating successful paths from the initial state to the goal state. Most of such problems are handled via repetition where the number of iterations is either increased to ensure successful path generation or decreased for faster results. 
In such a context, we have observed that random sampling-based approaches don't take into account repeated samples inside an obstacle. Such observation leads to wastage in budget based sample path planners where the number of samples is crucial for the path planner. In addition, some path planners will take into account the entire budget of samples for generating all possible paths from state to goal states. Again, even in such path planners, wasting away samples will only lead to worse results compared to the case where wastage is prevented; see RRT Dubins at Table~\ref{table:map_a}. To solve the path planning problem, one needs to come up with either an informative sampling technique or come up with a new path planner. Throughout the paper, we have chosen to be a plug-in component for path planners rather than providing new path planners.

\noindent\textbf{Discussion. } When presented with a limited number of sampling iterations, our proposed preprocessing technique ensured the existence of successful path generation over $20$ trials on multiple RRT-based planners, opposed to sampling ``blindly''; see Table~\ref{table:map_a} where $\infty$ represents the inability to generate a path.

\subsubsection{Method illustration} 
We refer the reader to Fig~\ref{fig:rrtfly} visualizing a classic run of RRT using our methods. Here, obstacles are bounded on the fly during the RRT run once we sample from them.

\begin{figure}[h]
    \centering
    \includegraphics[height=.2\linewidth, width=.8\linewidth]{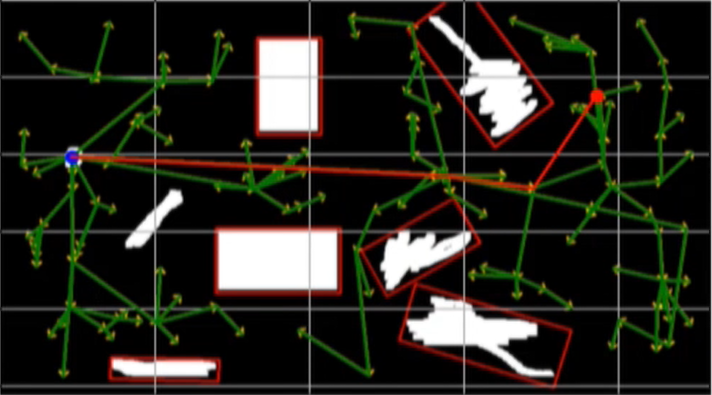}
    \caption{Running our methods on the fly with RRT.}
    \label{fig:rrtfly}
\end{figure}

\subsection{Bounding Convex Shapes}
To confirm our theoretical guarantees, we present the performance of our method for bounding convex shapes using an approximation towards the minimum volume enclosing ellipsoid. Our error $Err$ in which is stated in Algorithm~\ref{algfour} is shown as a function of the number of iterations. We note that the results shown in Fig.~\ref{fig:synthetic} have guaranteed almost an approximation of $1+\frac{1}{1001}$ to the \emph{MVEE}.

\begin{figure}[htb!]
    \centering
    \includegraphics[width=.9\linewidth,height=.8\linewidth]{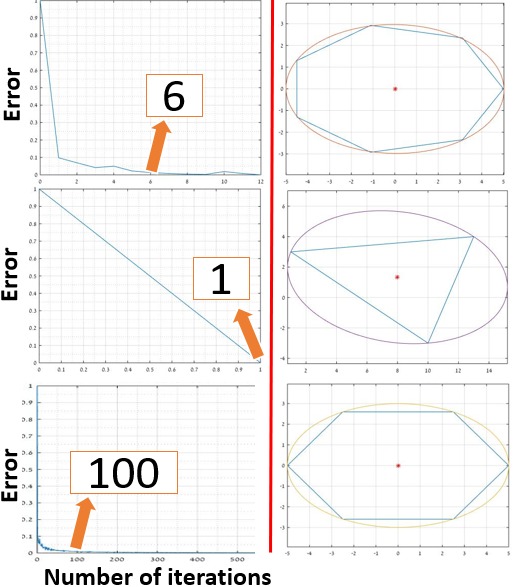}
    \caption{Our method's performance on convex polygons.}
    \label{fig:synthetic}
\end{figure}

\subsection{Map approximation}
We conclude our experiments by showing how our method can be used to generate an approximated map, i.e., $(1+\epsilon)$-approximation to the real map. The map is represented as a 2D binary map such that white pixels represent free space, and black pixels represent obstacle points; see Fig.~\ref{fig:univerise}. 
To highlight our performance for the task of map approximation, we have used a few known algorithms: \begin{enumerate*}[label=(\roman*)]
    \item The \emph{A*} algorithm, and
    \item the \emph{RRT} algorithm.
 \end{enumerate*}
To ensure the highest coverage of the map, we ran the $A^\ast$ algorithm where start and goal states to be the leftmost lower and rightmost upper corners of the map, respectively. As for the \emph{RRT} algorithm is run where the start state is at the center of the map; see Table~\ref{fig:map_approx}.

\begin{table}[htb!]
	\centering
	\caption{Comparison between $A^*$, RRT and our algorithm regarding a state space exploration problem. Here white/gray pixels denoted sampled points while black pixels denote unsampled points.}
\adjustbox{max width=\linewidth}{
  \begin{tabular}{|l||c|c|c|}
  \hline
    Algorithm &  
    ~5,000 queries & 
    ~10,000 queries & 
    sufficient queries
    \\ \hline

   A* (43K) &
     \includegraphics[trim=0 0 0 -20, scale = 0.1]{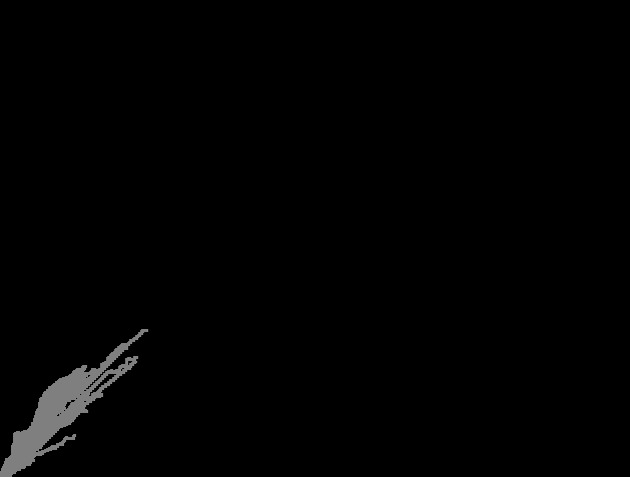} &
    \includegraphics[trim=0 0 0 -20, scale = 0.1]{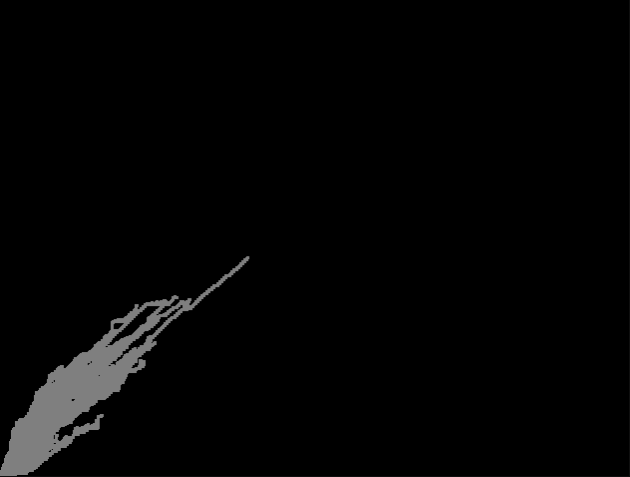} &
    \includegraphics[trim=0 0 0 -20, scale = 0.1]{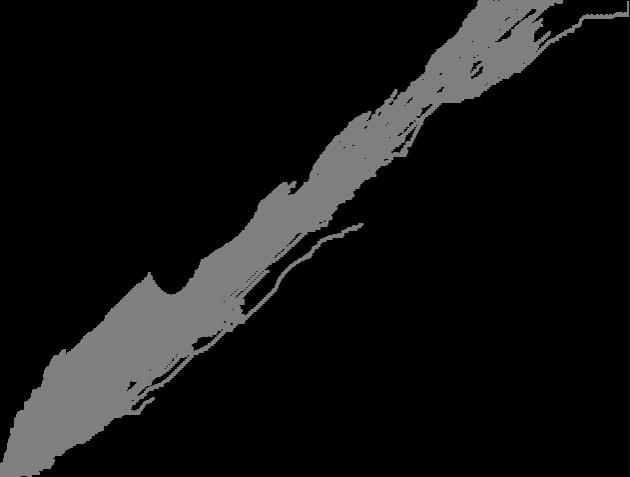} 
    \\ \hline
    
    RRT (15K)&
     \includegraphics[trim=0 0 0 -20, scale = 0.1]{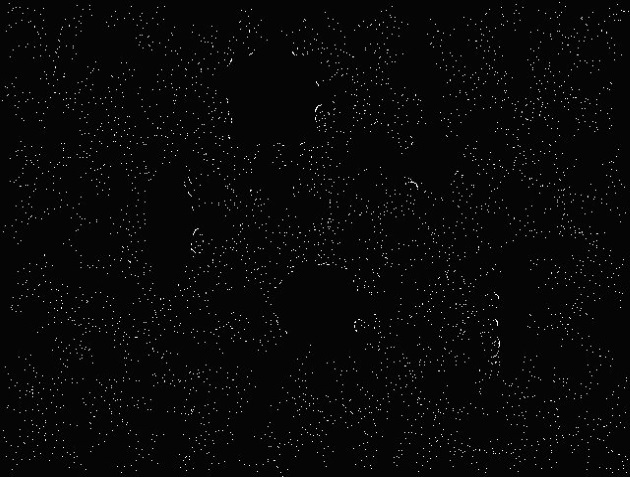} & \includegraphics[trim=0 0 0 -20, scale = 0.1]{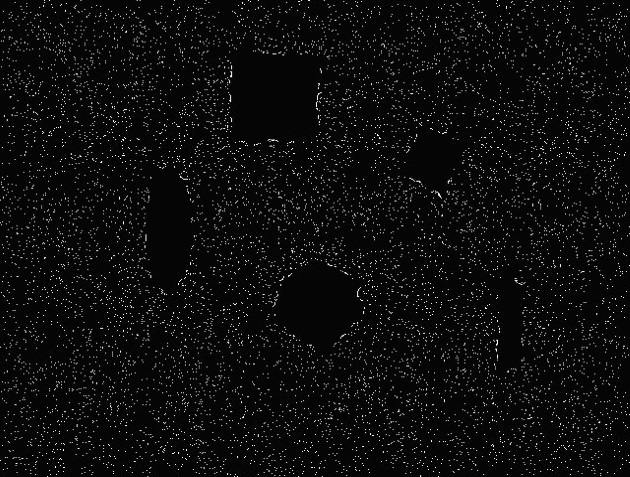} &
     \includegraphics[trim=0 0 0 -20, scale = 0.1]{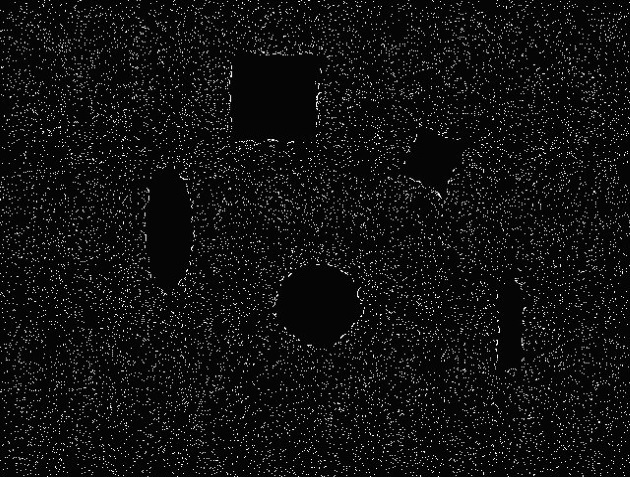}
    
    \\ \hline
    
    Our Alg. (14K)&
    \includegraphics[trim=0 0 0 -20, scale = 0.1]{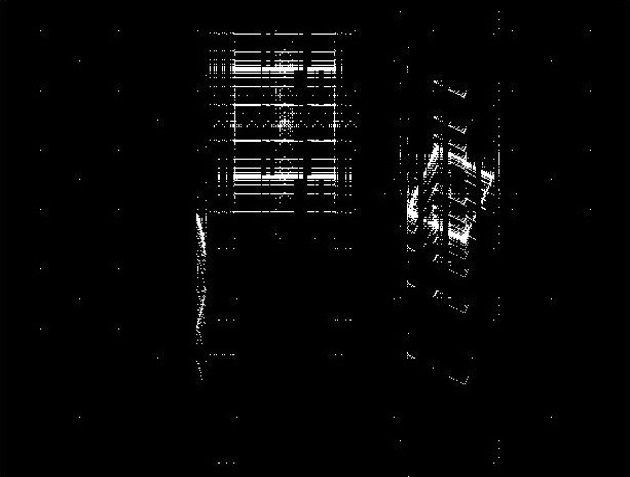} & \includegraphics[trim=0 0 0 -20, scale = 0.1]{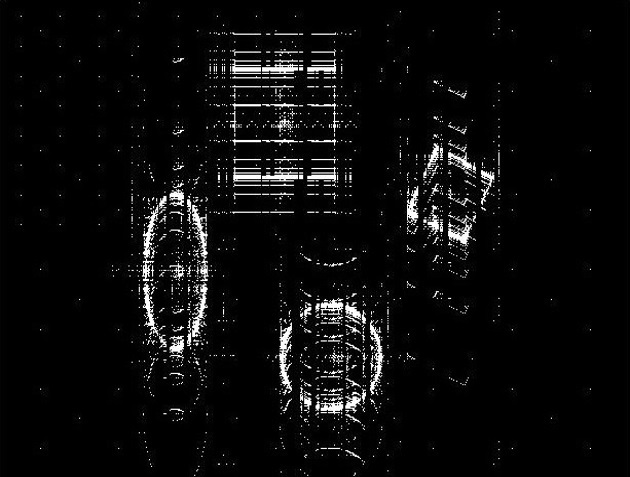} & \includegraphics[trim=0 0 0 -20, scale = 0.1]{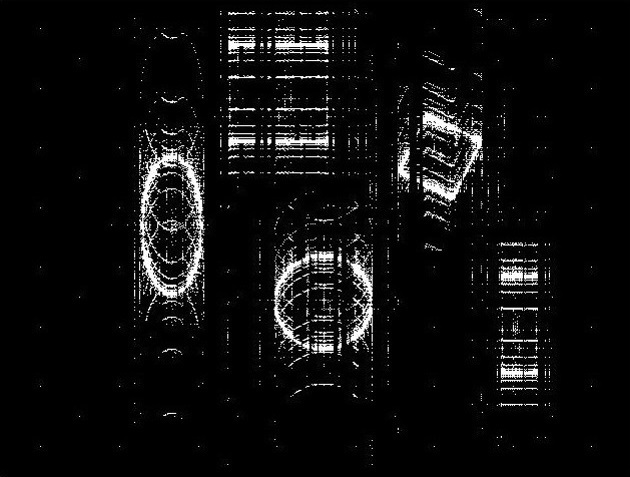}
    
    \\ \hline
    
    \hline
  \end{tabular}
  }
  \label{fig:map_approx}
\end{table}

\emph{$A*$} was not good enough to represent the real map as expected. On the other hand, RRT seems to be better at exploring the space of the map.
However, even after $15,000$ queries where white spots represent obstacle points, there are still many blind spots that we cannot determine whether they are obstacle points or not. Finally, we note that in less than $15,000$, our algorithm presented an almost perfect mapping of the world. 

\section{Proof Of Correctness}
In this section, we prove our results via the following.

\begin{definition}\cite{grotschel1988geometric} \label{def:cencircS}
Let $X \subseteq \REAL^d$ be a convex set, and $\epsilon > 0$ be a real number. Let $
S(K,\eps) = \br{x \in \REAL^d \mid \mynorm{x - y}_2 \leq \epsilon \text{ for some } y \in X},$ defines a ball of radius $\eps$ around $X$.
For a pair of positive real numbers $R > r > 0$, a positive integer $d \geq 2$, a convex body $X \subset \REAL^d$, and a point $a_0 \in X$, we denote by $(X;d,R)$ a circumscribed convex body such that $X$ is contained inside a ball of radius $R$ centered at the origin, and by $(X;d,R,r,a_0)$ a body that also contains a ball of radius $r$ centered at $a_0$; such body is referred to by the notion of \emph{centered body}.
\end{definition}

\begin{definition}\cite[Definition~2.1.14, Weak Membership Problem]{grotschel1988geometric} \label{def:memProb}
Let $X \subseteq \REAL^d$ be a convex set, then given a vector $y \in \REAL^d$ and a rational number $\delta > 0$, either,
\begin{enumerate*}
\item assert that $y \in S(X,\delta)$, or
\item assert that $y \not \in S(X,-\delta)$.
\end{enumerate*}
\end{definition}

\begin{definition}\cite[Definition~2.1.10, Weak Optimization Problem]{grotschel1988geometric}\label{def:optProb}
Let $X \subseteq \REAL^d$ be a convex set, then given a vector $c \in \REAL^d$ and a rational number $\varepsilon$, either
\begin{enumerate*}
\item finds a vector $y \in \REAL^d$ such that $y \in S(X,\varepsilon)$ and $c^Tx \leq c^Ty + \varepsilon$ for every $x \in S(K,-\varepsilon)$.
\item asserts that $S(X,-\varepsilon)$ is empty.
\end{enumerate*}
\end{definition}

\begin{definition}\cite[Definition~2.1.11, Weak Violation Problem]{grotschel1988geometric}\label{def:vioProb}
Let $X \subseteq \REAL^d$ be a convex set, then given a vector $c \in \REAL^d$, a rational number $\gamma$, and a rational number $\varepsilon > 0$, either
\begin{enumerate*}
\item assert that $c^Tx \leq \gamma + \varepsilon$ for all $x \in S(X,-\varepsilon)$ (i.e., $c^Tx \leq \gamma$ is almost valid), or
\item find a vector $y \in S(X,\varepsilon)$ with $c^Ty \geq \gamma - \varepsilon$ (a vector almost violating $c^tx \leq \gamma$).
\end{enumerate*}
\end{definition}

The following lemmas, help us in establishing our results:
\begin{lemma}\cite[Theorem~4.3.2]{grotschel1988geometric}\label{lem:memToVio}
There exists an polynomial time algorithm that solves the weak violation problem for every centered convex body $(X;d,R,r,a_0)$ given by a weak membership oracle, in $\tau = \left( \frac{dRr}{\varepsilon}\right)^{O(1)}$ oracle calls.
\end{lemma}

\begin{lemma}\cite[Remark~4.2.5]{grotschel1988geometric} \label{lem:vioToOpt}
There exists an oracle-polynomial time algorithm that solves the weak optimization problem for every circumscribed convex body $(X;d;R)$, given by a weak violation oracle $\tau = \left( \frac{dR}{\varepsilon}\right)^{O(1)}$.
\end{lemma}

Given a vector $u \in \REAL^d$ and $\varepsilon \in \REAL_+$, the following theorem shows that $\far$ yields a point which is far at most $\varepsilon$ from the farthest point in the convex set $X$ along $u$:

\begin{theorem}\label{thm:farProof}
Let $\varepsilon \in \REAL_+$ a real number and let $\oracle$ be a $\varepsilon$-weak membership oracle for a centered convex body $(X; d,R,r,a_0)$; see Definition~\ref{def:memProb} and Definition~\ref{def:cencircS}.
Let $u \in \REAL^d$ be a unit vector, $p \in \REAL^d$ an obstacle point, and let $\hat{x} \in X$ be the output of a call to $\far\term{\oracle, \eps, u, p}$.
Then the following hold:
\begin{enumerate*}[label={(\arabic*)}]
\item $\mynorm{\hat{x} - \arg\max_{x \in X} u^Tx}_2 \leq \epsilon$. \label{en:farGuar1}
\item The number of calls to the oracle is $M=\left(\frac{drR}{\varepsilon}\right)^{O(1)}$. \label{en:farGuar2}
\end{enumerate*}
\end{theorem}

\begin{proof}
The problem of finding the farthest point along a given direction in convex set, accessed implicitly via a polynomial membership oracle was addressed in~\cite{grotschel1988geometric} and is known as the optimization problem. Since we are dealing with bit-complexity problems, we are interested in the weaker version of optimization problem; See Definition~\ref{def:optProb}. By Plugging $\oracle$, $(X;d,R,r,a_0)$, $\epsilon$ into Lemma~\ref{lem:memToVio}, we obtain a weak violation oracle. Hence, plugging the resulted oracle in Lemma~\ref{lem:vioToOpt}, will attains a weak optimization oracle for a centered convex body $(X;d,R,r,a_0)$. We observe that by Definition~\ref{def:optProb} and Lemma~\ref{lem:vioToOpt}, plugging $u$ into $c$ and using $\epsilon$, will yield~\ref{en:farGuar1} and~\ref{en:farGuar2} at Theorem~\ref{thm:farProof}.
\end{proof}

\begin{theorem}\label{mainthm}
Let $\oracle$ be a membership oracle for a convex set $X\subseteq\REAL^d$; see Definition~\ref{def:memProb}.
Let $p\in \REAL^d$ and obstacle point and let $S \subseteq \REAL^d$ be the output of a call to $\mvecoreset(\oracle, \eps, \hat{S})$.
Then \begin{enumerate*}[label=(\roman*)]
\item $S$ is an $\eps$-coreset for the minimum volume enclosing ellipsoid (MVEE) of $X$, and 
\item if $\cfrac{\max_{x\in X}\mynorm{x}_2}{\min_{y\in X}\mynorm{y}_2}\leq r$, then $S$ can be computed in time $\tau=(\frac{dr}{\varepsilon})^{O(1)}$ and additional $\tau$ calls to $\oracle$.
\end{enumerate*}
\end{theorem}

\begin{proof}
First, Algorithm 4.2 in~\cite{todd2007khachiyan} computes a coreset $S$ and an ellipsoid $E$ as defined in Theorem~\ref{mainthm}, where $X$ is a finite set of $n$ points; see~\cite[Corollary 4.2]{todd2007khachiyan}. Our Algorithm~\ref{algfour} is the same up to few modifications: (i) we use the oracle to compute Algorithm~\ref{algfour}, which is a subroutine of Algorithm~\ref{algthree}. Algorithm~\ref{algfour} computes the farthest point in $X$ along a given direction $u$, i.e., $\max_{x\in X}u^T x$.
This problem can be solved in $O(\tau)$ time using membership oracle by combining Remark 4.2.5 and Theorem 4.3.2 from~\cite{grotschel1988geometric}. (ii) At Line~\ref{probOpt} of Algorithm~\ref{algfour}, we compute the farthest point $p^{+}_{i+1}\in X$ from the ellipsoid that is defined by the matrix $L_i$ and is centered at the point $c_i$. In~\cite{todd2007khachiyan} this was done using an exhaustive search over the finite set of points in $X$. In our case, we cannot use the oracle, as in case (i), since the desired function $\mynorm{L_i^T(p-c_i)}_2$ that we need to maximize over $p\in X$ is convex. In fact, this is a quadratic optimization over a positive-definite matrix, which is known to be NP-hard~\cite{murty1987some}. To this end, we use a relaxation and maximizes $\mynorm{L_i^T(p-c_i)}_1$, i.e., change the $\ell_2$ to $\ell_1$ norm. Since $\sqrt{d}\mynorm{x}_1\leq \mynorm{x}_2\leq \mynorm{x}_1$ for every $x\in\REAL^d$, we get a $\sqrt{d}$ approximation. The result is a mixed integer convex optimization problem that we can solve, obtaining an approximate solution. We now prove that we may change in Algorithm 4.2 in~\cite{todd2007khachiyan}, where we replace the farthest point by a point which may not be the farthest, but only up to a factor of $\sqrt{d}$ and still get the same result. The only difference is that the number of iterations increases by a factor of $d$. 

Indeed, let
$p \in \argmax_{x \in X} \mynorm{L^Tx}_2$, and $ p^\prime \in \argmax_{x \in X} \mynorm{L^Tx}_1$, 
where $L \in \mathbb{R}^{d \times d}$ such that $Q = LL^T$. Our proof is essentially a variant of the original proof in~\cite{kumar2005minimum}. If $p^\prime = p$ then we have found the desired point. Otherwise, denote $y^\prime =  L^Tp^\prime$ and $y =  L^Tp$. By the properties of $\ell_p$ norms, we have
$\mynorm{y^\prime}_2 \leq \mynorm{y}_2 \leq \mynorm{y}_1 \leq \mynorm{y^\prime}_1 \leq \sqrt{d} \cdot \mynorm{y^\prime}_2.$
Hence,$ \frac{\mynorm{y}_2}{\sqrt{d}} \leq \mynorm{y^\prime}_2 \leq \mynorm{y}_2$. Setting $\tilde{p} := \sqrt{d} p^\prime$ proves that $\tilde{p}$ is an approximation to the farthest point. Hence using the previous inequality and (36) of~\cite{kumar2005minimum}, we have
\begin{equation}
\label{KtoEps}
k_i \leq \tilde{k}_i \leq d \cdot k_i,
\end{equation}
where $k_i = \mynorm{y}_2^2$, $\tilde{k}_i = \left\| \sqrt{d} \cdot \tilde{y}\right\|^2$. We need to compute $\tilde{\varepsilon_i}$. To do so, let $\alpha_i \geq 0$ such that $\tilde{\varepsilon_i} = \alpha_i \cdot \varepsilon_i$, and we will establish upper and lower bounds on $\alpha_i$, using~\eqref{KtoEps}. By the left side of~\eqref{KtoEps}, $k_i = (1 + d) \cdot (1 + \varepsilon) \leq \tilde{k}_i =(1+d) \cdot (1+\alpha \varepsilon) \Rightarrow \alpha \geq 1.$ Using the right side of~\eqref{KtoEps}, yields an upper bound on $\alpha_i$, $
\tilde{k}_i = (1 + d) \cdot (1 + \alpha_i \varepsilon_i) \leq d \cdot k_i = d (1+d) \cdot (1 + \varepsilon_i)$. Thus, $\alpha_i \leq \frac{d \cdot (1 + \varepsilon_i) - 1}{\varepsilon_i}$
By~\cite{kumar2005minimum}, the ellipsoid method halts when the following inequality holds $\varepsilon_i \leq (1 + \varepsilon)^{\frac{2}{d+1}} - 1.$ Hence, $\alpha_i \leq d$.  We have computed $\tilde{k}_i, \tilde{\varepsilon}$ to compute $\tilde{\beta}_i$. This term denotes the step size used to update the weights of the points. By (37) of~\cite{kumar2005minimum} $\tilde{\beta}_i = \frac{\tilde{k}_i - (d + 1)}{(d + 1) \cdot (\tilde{k}_i - 1)} 
= \frac{(d+1)\cdot (1 + \tilde{\varepsilon}_i) - (d + 1)}{(d + 1) \cdot (\tilde{k}_i - 1)} 
= \frac{\tilde{\varepsilon}_i}{\tilde{k}_i - 1}$.
Let $v_i$ denote the logarithm of the volume of the ellipsoid at the $i^{th}$ iteration.
Hence, by plugging the previous equality into (40) of~\cite{kumar2005minimum}, we obtain
$v_{i+1} = v_i + d \cdot \log{\left( 1 - \tilde{\beta}_{i}\right)} + \log{\left( 1 + \tilde{\varepsilon}_i \right)} 
= v_i + d \cdot \log{\left( 1 - \frac{\tilde{\varepsilon}_i}{\tilde{k}_i - 1} \right)} + \log{\left( 1 + \tilde{\varepsilon}_i \right)} 
= v_i + d \cdot \log{\left( \frac{d \cdot (\tilde{\varepsilon}_i + 1)}{d \cdot (\tilde{\varepsilon}_i + 1) + \tilde{\varepsilon}_i}\right)} + \log{\left( 1 + \tilde{\varepsilon}_i \right)} 
= v_i - d \cdot \log{\left(1 + \frac{\tilde{\varepsilon}_i}{d \cdot (\tilde{\varepsilon}_i + 1)} \right)} + \log{\left( 1 + \tilde{\varepsilon}_i \right)} $
Since $\tilde{\varepsilon}_i \geq 0$, we obtain that
\begin{equation}
\label{EllipsVolPerIter}
\begin{gathered}
v_{i+1} \geq v_i - \frac{\tilde{\varepsilon}_i}{d \cdot (\tilde{\varepsilon}_i + 1)} + log{\left( 1 + \tilde{\varepsilon}_i \right)} \\
\geq v_i +
\begin{cases}
log(2) - \frac{1}{2} & \tilde{\varepsilon}_i \geq 1 \\
\frac{\tilde{\varepsilon}_i^2}{8} & \tilde{\varepsilon}_i < 1
\end{cases}
\end{gathered}
\end{equation}
where the inequality is based on $\log{\left( 1 + x\right)} \leq x$ where $x > -1$. By~\eqref{KtoEps} and (41) in~\cite{kumar2005minimum}, we obtain that $\tilde{k}_0 \leq d\cdot k_0 \leq d (d+1) n$. Thus, $\tilde{\varepsilon}_0 \leq dn - 1$. Plugging this inequality and~\eqref{EllipsVolPerIter} into (42) of~\cite{kumar2005minimum}, yields 
\begin{enumerate*}[label=(\roman*)]
    \item $v_0 \geq -\infty$, 
    \item $v^\ast - v_i \leq (d+1) \log{\left( 1 + \tilde{\varepsilon}_i\right)}$, 
    \item $v_{i+1} - v_i \geq log{\left( 1 + \tilde{\varepsilon}_i \right)} -\frac{\tilde{\varepsilon}_i}{d \cdot (\tilde{\varepsilon}_i + 1)}$, and 
    \item \label{ourDelta}
$\delta_0 = v^\ast - v_0 \leq (d + 1) \cdot (\log{n} + \log{d})$. 
\end{enumerate*}
Hence by substituting $\eps$ with $d \cdot \eps$ in (44) in~\cite{kumar2005minimum}, we yield the desired approximation and the number of iterations needed is $\mathcal{O}\left( \frac{1}{\epsilon} \right)$. Plugging~\ref{ourDelta} in (43) and  in~\cite{kumar2005minimum}, yields that the maximum number of iterations, $K$, needed until the ellipsoid method converges is
$
K = d \cdot \log{\delta_0} \in O\left( d \cdot \left(\log{(d + 1)} + \log{(\log{n} + \log{d})}  + \frac{1}{\epsilon} \right) \right).$
\end{proof}

\section{Conclusion}
We suggested a novel preprocessing technique that discovers obstacles in a map, to remove redundancies from the sampling space, and thus improve the running time and/or the final path length of different RRT-based planners. Such preprocessing step is done once. We bound each obstacle by its minimum enclosing ellipsoid once a point is sampled from it. Thus, one can find the smallest simplex which contains this ellipsoid, to exclude it from the sampling space. Following this step, a novel sampling technique is performed via the constrained Delaunay triangulation. Each of these steps is theoretically motivated, and supported by theorems and proofs. Finally, the experimental results match the theoretical contribution where the performance was clearly improved on a variate of space sampling based algorithms.

\bibliographystyle{abbrv}
\bibliography{main}
\end{document}